\documentclass{article}

\usepackage{microtype}
\usepackage{graphicx}
\usepackage{subfigure}
\usepackage{booktabs} 

\usepackage[utf8]{inputenc} 
\usepackage[T1]{fontenc}    

\usepackage{url}            
\usepackage{booktabs}       
\usepackage{amsfonts}       
\usepackage{nicefrac}       
\usepackage{microtype}      
\usepackage{xcolor}         
\usepackage{amsmath}
\usepackage{amssymb}
\definecolor{olive}{rgb}{0.6, 0.6, 0.2}
\definecolor{sand}{rgb}{0.8666666666666667, 0.8, 0.4666666666666667}
\definecolor{wine}{rgb}{0.5333333333333333, 0.13333333333333333, 0.3333333333333333}
\usepackage{MnSymbol,bbding,pifont}
\usepackage{wrapfig}

\usepackage{fontawesome5}
\usepackage{amsthm}

\definecolor{cvprblue}{rgb}{0.21,0.49,0.74}
\usepackage[pagebackref,breaklinks,colorlinks,citecolor=cvprblue]{hyperref}

\definecolor{electricindigo}{rgb}{0.44, 0.0, 1.0}
\usepackage{multirow}

\newcommand{\Jiahao}[1]{{\color{green} #1}}

\definecolor{deblue}{RGB}{11,132,147}
\definecolor{ocra}{RGB}{204, 119, 34}
\usepackage{enumitem}
\usepackage{tikz}
\newcommand{\fcircle}[2][red,fill=red]{\tikz[baseline=-0.5ex]\draw[#1,radius=#2] (0,0.03) circle ;}

\newcommand{\PhysGTN}{\textcolor{black}{{\fontfamily{cmtt}\selectfont HAMLET}}}

\usepackage[accepted]{icml2024}

\usepackage{amsmath}
\usepackage{amssymb}
\usepackage{mathtools}
\usepackage{amsthm}
\usepackage{color}
\usepackage{colortbl}
\usepackage[capitalize,noabbrev]{cleveref}

\theoremstyle{plain}
\newtheorem{theorem}{Theorem}[section]
\newtheorem{proposition}[theorem]{Proposition}

\theoremstyle{definition}
\newtheorem{definition}[theorem]{Definition}

\theoremstyle{remark}

\usepackage[textsize=tiny]{todonotes}

\icmltitlerunning{\PhysGTN\-: Graph Transformer Neural Operator for Partial Differential Equations}

\begin{document}

\twocolumn[
\icmltitle{\PhysGTN\-: Graph Transformer Neural Operator for Partial Differential Equations}
%



\icmlsetsymbol{equal}{*}

\begin{icmlauthorlist}
\icmlauthor{Andrey Bryutkin}{equal,mmm}
\icmlauthor{Jiahao Huang}{equal,comp}
\icmlauthor{Zhongying Deng}{yyy}
\icmlauthor{Guang Yang}{comp}
\icmlauthor{Carola-Bibiane Schönlieb}{yyy}
\icmlauthor{Angelica Aviles-Rivero}{yyy}
\end{icmlauthorlist}

\icmlaffiliation{yyy}{Department of Applied Mathematics and Theoretical Physics, University of Cambridge,  UK}
\icmlaffiliation{comp}{Bioengineering Department and Imperial-X,  National Heart and Lung Institute  \& Cardiovascular Research Centre, Imperial College London, UK.}
\icmlaffiliation{mmm}{Department of Mathematics, MIT,  USA}

\icmlcorrespondingauthor{Andrey Bryutkin}{bryutkin@mit.edu}
\icmlcorrespondingauthor{Jiahao Huang}{j.huang21@imperial.ac.uk}

\icmlkeywords{PDEs, Neural Operator, Deep Learning}

\vskip 0.3in
]



\printAffiliationsAndNotice{\icmlEqualContribution} 

\begin{abstract}
We present a novel graph transformer framework, \PhysGTN, designed to address the challenges in solving partial differential equations (PDEs) using neural networks. The framework uses graph transformers with modular input encoders to directly incorporate differential equation information into the solution process. This modularity enhances parameter correspondence control, making \PhysGTN\ adaptable to PDEs of arbitrary geometries and varied input formats. Notably, \PhysGTN\ scales effectively with increasing data complexity and noise, showcasing its robustness. \PhysGTN\ is not just tailored to a single type of physical simulation, but can be applied across various domains. Moreover, it boosts model resilience and performance, especially in scenarios with limited data. We demonstrate, through extensive experiments, that our framework is capable of outperforming current techniques for PDEs. 
\end{abstract}

\section{Introduction}\label{sec:introduction}%
Deep learning has revolutionised all domains, including the solution of partial differential equations (PDEs). Traditional numerical methods for solving PDEs can be computationally expensive, especially for high-dimensional problems or those with complex geometries. In recent years, a number of approaches have been developed to leverage the power of deep neural networks for PDEs, including Physics-Informed Neural Networks (PINNs) e.g.~\cite{raissi_physics-informed_2019,karniadakis2021physics_informed_machine_learning} and Neural Operators e.g.~\cite{li_fourier_2022,lu_deeponet_2021}. These methods use neural networks to learn the solution to a PDE directly from the data, without the need for discretisation or meshing. Deep learning for PDEs is largely changing many fields of science and engineering, from fluid dynamics and electromagnetics to finance and healthcare~\cite{liu2019multi,kissas2020machine,sahli2020physics,sun2020surrogate,raissi2019deep}.

Although PINNs and their variants have shown success in solving PDEs, they still have limitations. These include limited generalisability across multiple PDE instances, a lack of discretisation invariance, and an inability to generalise beyond a specific resolution/geometry observed during training~\cite{wang2021understanding,fuks2020limitations}. A body of researchers has investigated alternatives to mitigate such limitations, where principles driven by neural operators are considered. These neural operators directly approximate the differential operator that underlies the PDE and have yielded promising outcomes in recent years, especially for multi-instance learning. However, significant research has yet to be conducted in this area.


\textbf{Contributions.} We present a novel framework called grap\textbf{H} tr\textbf{A}nsfor\textbf{M}er neura\textbf{L} op\textbf{E}ra\textbf{T}or -- \PhysGTN, which aims to address the challenges in solving PDEs using neural networks. 
\PhysGTN\ stands out as the first framework of its kind, which employs graph transformers with modular input encoders. Choosing a graph neural network as a foundational architecture enables the capture of complex interactions and dynamics more effectively than standard transformers due to its ability to model intricate spatial relationships and nonlinear interactions efficiently. This capability is especially crucial in handling rapid changes and high-frequency information at initial time steps, as demonstrated by our experimental results. 
\PhysGTN\ has adaptability to irregular meshes, allowing it to solve discretisation-invariant PDEs. Furthermore, \PhysGTN\ showcases exceptional performance even when faced with limited data, it is meticulously designed to address the current challenges in the field, underscoring a significant advancement over existing approaches. Notably, we emphasise:
%

\fcircle[fill=deblue]{2pt} We introduce \textit{the first graph transformer architecture} (\PhysGTN) for PDEs, in which we highlight: \vspace{-0.1cm}
\begin{itemize} [noitemsep,nolistsep]
    \item[--] Our  \PhysGTN\ framework can effectively handle PDEs with arbitrary geometries and diverse input formats, including point non-uniform meshes and design parameters.
    \item[--] \PhysGTN\ provides a strong graph perspective, a feature that significantly boosts the resilience and performance of the model, especially when dealing with limited data availability.    
    \item[--] \PhysGTN\ learns solution operators on a flexible grid, which enhances the computational efficiency of transformers and the flexibility of graphs for problem solving.
\end{itemize}
\fcircle[fill=deblue]{2pt}  We extensively evaluate \PhysGTN\ through experiments on various graphs, which enhances its inference characteristics and reduces overfitting. In addition, we compare our proposed technique with existing approaches and validate it on multiple datasets.

\section{Related Work}\label{sec:related_work}
\fcircle[fill=wine]{2.5pt} \textbf{Learned PDEs}. New methods have been proposed in various fields, such as fluid dynamics for solving PDEs \cite{brunton2020machine, stachenfeld2022learned_simulations}. These advances have led to the extension of architectures that describe continuous-time solutions and multiparticle dynamics \cite{iakovlev_learning_2021, wang2020towards}. With knowledge of the underlying PDEs, physics-informed models can be used to retrieve single-instance solutions in both unsupervised and semi-supervised settings \cite{raissi_physics-informed_2019, zhu2019physics, karniadakis2021physics_informed_machine_learning}.
The standard strategy for solving such tasks is to encode the data spatially and then use various schemes to evolve over time. Deep learning architectures employ models such as convolutional layers \cite{ronneberger2015unet, zhu_bayesian_2018, bhatnagar2019prediction_convolutional}, symbolic neural networks \cite{Kaiser_2018_sindy, long2019pde}, residual networks \cite{he2016deep}, and more advanced physics-informed networks~\cite{jagtap2021extended}, as well as methods known from finite element methods (FEM) such as Galerkin and Ritz~\cite{sirignano2018dgm, e2018deep_ritz}.

\fcircle[fill=wine]{2.5pt} \textbf{Neural Operators}. After the proposal of DeepONet \cite{lu_deeponet_2021} as a general operator learning framework, further research was carried out to learn the underlying solution operator using model reduction techniques~\cite{bhattacharya2021model}. Neural operators, such as graph neural networks~\cite{li2020multipole, li2020neural} or convolutional layers in the Fourier space~\cite{li2020fourier}, are models that describe functions between infinite-dimensional spaces. The Fourier Neural Operator (FNO) and its variants, such as incremental, factorised, and adaptive FNO or FNO+~\cite{zhao2022incremental, tran2021factorized, guibas2022adaptive, li_fourier_2022}, have demonstrated remarkable results in terms of speed and error, with their main advantage being discretisation invariance. 
Most recently, new adaptive neural operator models have been proposed that use transformations~\cite{tripura2022wavelet, gupta2021multiwavelet}, novel message-passing algorithms on the graph~\cite{brandstetter2022message, boussif2022magnet, ong_iae-net_2022}, attention mechanisms~\cite{li_transformer_2022, kissas_learning_2022}, or different function bases~\cite{fanaskov_spectral_2022}. To improve performance, methods with increasing physical inductive bias~\cite{li2021physics, wang2021learning}, data augmentation~\cite{brandstetter2022lie}, and symmetry-related approaches~\cite{brandstetter2022clifford} have been proposed.

\fcircle[fill=wine]{2.5pt} \textbf{Graph Neural Networks in Physical Systems}.
GNNs have demonstrated their advantage in analysing complex dynamical systems by describing interactions on irregular grids \cite{li2022graph_lagrangian_fluids, battaglia_interaction_2016, battaglia_relational_2018, sanchez-gonzalez_graph_2018, alet_graph_2019}. Message-passing algorithms in GNNs have been used to pass messages between small groups of moving or interacting objects \cite{gilmer_neural_nodate}. New research has been conducted on solving differential equations with GNNs, but spatial information about the location of nodes is needed and there is increased computational complexity \cite{li2020neural, li2020multipole}. GNNs have been used to solve complex initial and boundary value problems and to model continuous-time propagation \cite{lotzsch_learning_2022, horie_physics-embedded_nodate, pilva_learning_2022, iakovlev_learning_2021}. All-in-one models have also been proposed by describing PDEs as neural operators \cite{brandstetter2022message, li2020neural}.

\fcircle[fill=wine]{2.5pt} \textbf{Transformers in Physical Systems}. The attention mechanism has been successful in describing dynamics and patterns in physical systems \cite{geneva_transformers_2022, han2022predicting}. Attention layers capture structures and patterns in the spatial domain of PDEs \cite{shao2022sit, cao_choose_nodate}, while temporal evolution is modelled using attention \cite{geneva_transformers_2022, song_m2n_nodate}. Our model uses attention-based layers for spatial encoding and a recurrent multilayer perceptron (MLP) for time marching in latent space. Graph neural networks have also been studied using attention \cite{min_transformer_2022, yun_graph_nodate, dwivedi_generalization_2021}, and architectures are used to model the solution operator for PDEs \cite{bo2023specformer, li_transformer_2022, cao_choose_nodate, kissas_learning_2022}. However, to the best of our knowledge, graph transformers and their capabilities have not yet been explored for PDEs, opening the door to a new research line. Janny et al. \cite{janny2023eagle} introduce a mesh-based transformer (EAGLE) optimising for large distances through pregenerated coarse meshes, combining GNN and self-attention for dynamic data. \PhysGTN\ contrasts with its graph transformer encoder, CrossFormer for parameter integration, and MLP-based time propagation, which favours recurrent over autoregressive model updates. Steeven et al. \cite{steeven2024space}, parallel in aim, differs in topology and integration methods, using a GNN backbone and recurrent updates with continuous integration, unlike \PhysGTN\'s single CrossFormer approach.
\begin{figure*}[t!]
  \centering
  \includegraphics[width=1\textwidth]{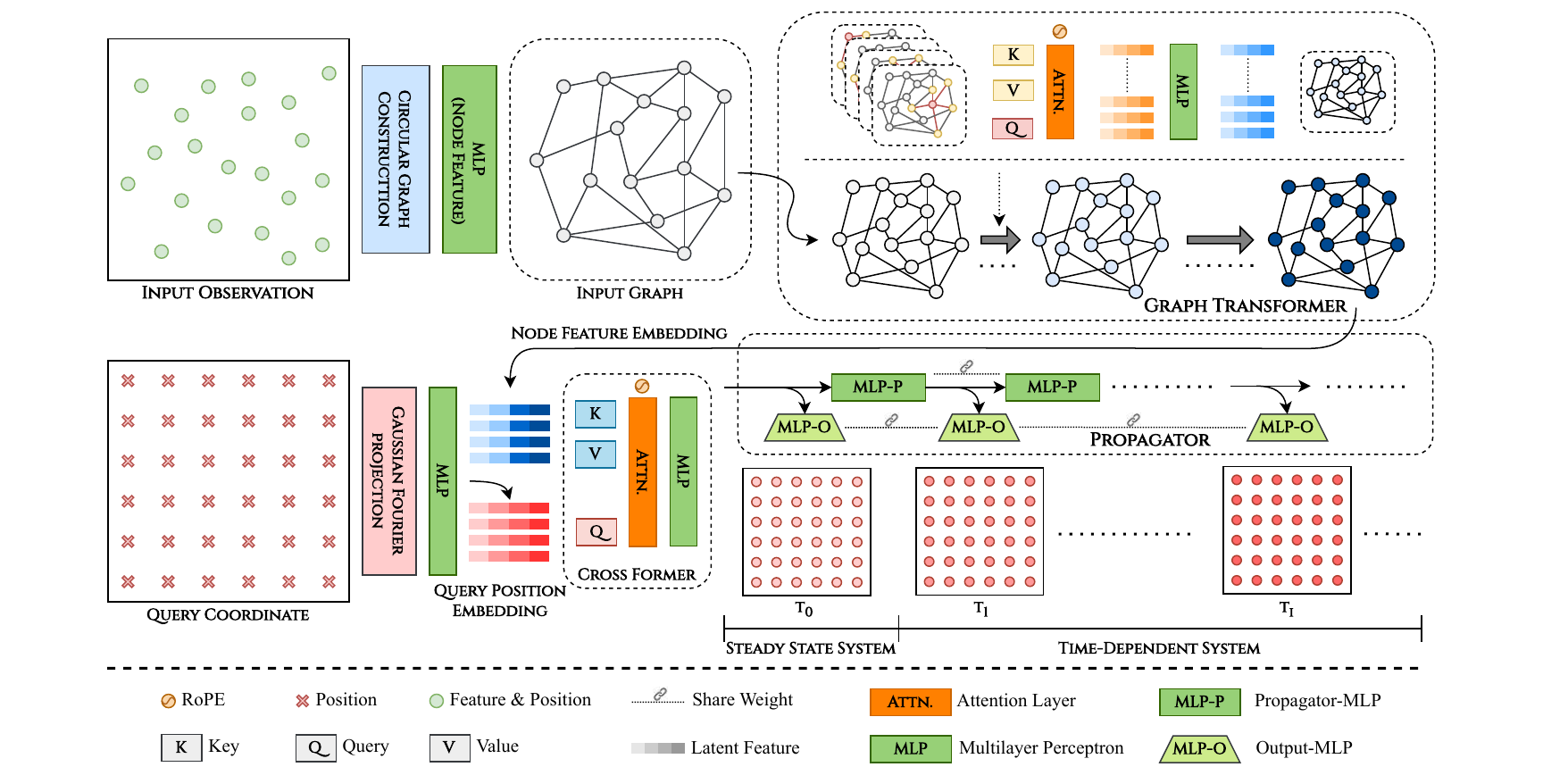}
  \caption{
  Architectural overview of \PhysGTN. 
  The diagram depicts the transformation of input observations into graph representations, followed by node feature embedding using a multilayer perceptron (MLP) to prepare key, value and query vectors for the graph transformer. 
  The transformer captures data structure and node relationships, which are then refined by the cross-former with query positions. Finally, shared MLPs propagate the embeddings to model system behaviour over time.
  }
  \label{fig:FIG_TEASER}
\end{figure*}

\section{Methodology: \PhysGTN}\label{sec:methodology}

\textbf{Problem statement.} 
We consider parametric partial differential equations defined on a domain $\mathcal{D}\subset\mathbb{R}^n$. We denote the domain by $\mathcal{D}_\theta$, where $\theta\in\mathcal{\mathcal{P}}$ and $\mathcal{P}\subset\mathbb{R}^p$ is the parameter space. Thus, the domain $\mathcal{D}_\theta$ is parameterised by $\theta$, which is sampled over the domain according to the distribution $\nu$. The general form of our problems reads:
\begin{equation}
\label{eq:DefinitionPartialDifferentialEquation}
\begin{aligned}
    P&: \mathcal{P}\times\mathcal{D} \times \mathcal{V} \times \mathbb{R}^m \times \ldots \times \mathbb{R}^m \rightarrow \mathbb{R}^{\ell}, \quad \mathcal{D} \subset \mathbb{R}^n, \mathcal{V} \subset \mathbb{R}^m, \\
    P&\left(\theta, x, u, \partial_{x_1} u, \ldots \partial_{x_n} u, \ldots, \partial_{x_1}^{\alpha_1} \cdots \partial_{x_n}^{\alpha_n} u\right)=0.
\end{aligned}
\end{equation}
The fixed map $P$ is solved by a vector-valued function $u: \mathcal{D} \rightarrow \mathcal{V}$, which is not known a priori. Here $\alpha$ is a multiindex defined as $\alpha=\left(\alpha_1, \ldots \alpha_n\right)$, with $\alpha_i \in \mathbb{N}_0,|\alpha|=\sum_{i=1}^n \alpha_i$. In our case, the first entry represents the time; the domain reduces to $\mathcal{T}\subset\mathbb{R}_{\geq 0}$ and $\mathcal{D}_\theta\subset\mathbb{R}^{n-1}$, if the equation depends on additional space variables.
For a well-defined setting, for most of the problems, at least one of the following conditions should be fulfilled:
\begin{subequations}
\begin{align}
    u(x)=u_0(x), & \quad x\in \mathcal{D}_\theta \times\{T_0\}, \label{eq:PDE_IVC}\\
    u(x)=u_b(x), & \quad x\in \partial \mathcal{D}_\theta \times \mathcal{T}, \label{eq:PDE_BVC}
\end{align}
\end{subequations}
where $T_0\in\mathcal{T}$. We call $u_0$ the initial condition and $u_b$ the boundary condition. Assume $\mathcal{P}, \mathcal{D}, \mathcal{V}$ to be Banach spaces with a given norm. We assume that there exists an analytic solution operator: $\mathcal{S}: \mathcal{P}\times\mathcal{D}\times\mathbb{R}^m\times\ldots\times\mathbb{R}^m\times\mathbb{R}^\ell\times\mathbb{R}^\ell\rightarrow\mathcal{V}$, $(\theta, x, \partial_{x_1} u, \ldots \partial_{x_n} u, \ldots, \partial_{x_1}^{\alpha_1} \cdots \partial_{x_n}^{\alpha_n} u, u_b, u_0)\mapsto u$. We then seek to design an architecture to function as an approximate solution operator, denoted by $\tilde{\mathcal{S}}_\mu: (\theta, u_0,u_b)\mapsto u$, where $\mu\in\mathbb{R}^p$ represents the parameters of the neural network. 
We can assume that the dataset $\{\theta^{(n)},u^{(n)} \}^N_{n=1}$ is provided, where $\theta^{(n)} = \{\theta^{(n)}_j\}_{j=1}^L$ is the collection of input parameters describing the system's state on a $L$-discretised domain. We use the notation $\theta^{(n)}=\theta(x^{(n)})$ and $u^{(n)}=u(x^{(n)})$ for brevity. Furthermore, we assume that our operator $\tilde{\mathcal{S}}_\mu(\theta^{(n)})=u^{(n)}$ may be corrupted with noise. We underline that in this work, we consider both type of problems: stationary and time-dependent. Refer to Figure~\ref{fig:FIG_TEASER} for an overview of our \PhysGTN\ model.

\subsection{Graph Construction for \PhysGTN}\label{sec:methodology_graph}

\textbf{Input Format for a Graph Space.} 
The input of \PhysGTN\ is defined on a general graph $\mathcal{G}$ composed of a set of nodes $V$ and edges $E$. Each node in the graph also has predefined node features that are specific to the PDE problem being considered. 
The model can accept various input parameter formats. In many cases, the input is given by a specific mesh design that discretises the spatial domain $\mathcal{D}_\theta$ into $L$ nodes, which we call an \textit{$L$ point discretisation} of the domain $\mathcal{D}_\theta$. 

\textbf{Graph Construction.} 
We define a graph $\mathcal{G}=(V, E)$ with a set of nodes $V=\{h_i\}^L_{i=1}$ and undirected edges $E=\{e_{ij}\}$. The node feature vectors $h_i\in\mathbb{R}^d$ represent the attributes of each node, where $d$ is the dimension of the feature vector. Edge features are properties of connections between nodes, which in our case are used to generate messages in the transformer architecture. 
To define the edges on the graph, we use a 
truncation function, which creates the edges by assigning each node $i$ its neighbourhood $\mathcal{N}_i$.


\textit{How can we generate the edges between the nodes for the connectivity of the graph?}
To generate the edges between the nodes of the graph, we treat the $L$-point discretisation $\{x_i\}^L_{i=1} \subset \mathcal{D}_\theta$ as the nodes and use a truncation function $s$ to map the nodes to subsets of $\mathcal{D}$ denoted $\mathcal{B}(\mathcal{D}_\theta)$. 
We assign and concatenate the parameter value $\theta_i$ and the corresponding position $x_i$ to each node, where the node feature can be defined as $h_i = \theta_i \| x_i$, ($\|$ is concatenation). Connectivity between nodes is defined by the intersection of a node's image under $s$ with the set of all nodes $V$. These subsets of nodes $\mathcal{B}(\mathcal{D}_\theta)$ are called the neighbourhoods of the nodes. In our case, the nodes are connected by an edge if the Euclidean distance between them is less than a predefined radius $r$, and the neighbours are distributed within a circular area. If $s(x)=\mathcal{D}_\theta$, the graph is fully connected. Our choice of using a connectivity ball generator as $s$ leads to a sparse graph, which reduces computational cost.

\subsection{Graph Transformer for \PhysGTN}\label{sec:methodology_graph_transformer}

\textbf{Graph Transformer Block.} 
Graph transformers~\cite{dwivedi_generalization_2021} are neural networks that process graph-structured data, allowing them to learn robust node representations and representations of the entire graph. 
The formulation of the graph-based multi-head self-attention mechanism enables efficient message passing on a graph.
\begin{equation}
\label{eq:GraphAttention}
\begin{aligned}
    \hat{h}_i^{\ell} &= O_h^{\ell} \|_{k=1}^H\left(\frac{1}{|\mathcal{N}_i|}\sum_{j \in \mathcal{N}_i} w_{i j}^{k, \ell} V^{k, \ell} h_j^{\ell}\right), \\ 
    \text { where, } w_{i j}^{k, \ell} &= \operatorname{softmax}_j\left(\frac{Q^{k, \ell} h_i^{\ell} \cdot K^{k, \ell} h_j^{\ell}}{\sqrt{d_k}}\right),
\end{aligned}
\end{equation}
and $Q^{k, \ell}, K^{k, \ell}, V^{k, \ell}\in\mathbb{R}^{d_k \times d}$, $O_h^{\ell}\in\mathbb{R}^{d \times d}$ are fully-connected layers at $k^\text{th}$ attention head in the $l^\text{th}$ transformer block. $\|$ denotes concatenation. 
We integrate position information using Rotary Position Embedding (RoPE)~\cite{su2021roformer}, even though the node features contain positional information unlike vector-based multi-head self-attention. For brevity, the RoPE operator is included into $Q^{k, \ell}:= R_{\Theta} \bar{Q}^{ k, \ell}$ and $K^{k, \ell}:= R_{\Theta} \bar{K}^{ k, \ell}$.

In the message-passing algorithm framework, the attention layer acts as a message that aggregates data from neighbouring nodes. Normalisation is crucial before summing to ensure alignment with neural operator classes. The feedforward neural network takes node features from the latent vector after the multi-head self-attention layer: 
\begin{equation}
\label{eq:GraphTransformerBlock}
\begin{aligned}
    &\hat{h}_i^{\ell} = \operatorname{Attn}(h_i^{\ell}), \
    &\hat{\hat{h}}_i^{\ell} = \operatorname{Norm}(h_i^{\ell}+\hat{h}_i^{\ell}), \\
    &\hat{\hat{\hat{h}}}_i^{\ell} = W_2^{\ell}\cdot \operatorname{ReLU}(W_1^{\ell}  \hat{\hat{h}}_i^{\ell}), \
    &h_i^{\ell+1} = \operatorname{Norm}(\hat{\hat{h}}_i^{\ell}+\hat{\hat{\hat{h}}}_i^{\ell}), 
\end{aligned}
\end{equation}
where $W_1^{\ell}, \in \mathbb{R}^{2 d \times d}$ and $W_2^{\ell}, \in \mathbb{R}^{d \times 2 d}$ are fully-connected layers. $\hat{h}_i^{\ell}$, $\hat{\hat{h}}_i^{\ell}$ and $\hat{\hat{\hat{h}}}_i^{\ell}$ denote intermediate representations. $\operatorname{Norm}$ can be either layer normalisation \cite{ba2016layer} or batch normalisation \cite{ioffe2015batch}.
The attention layer acts as the message $m_i$. 

\textbf{\PhysGTN\ corresponds to the class of Neural Operators.} The operator definition (see Appendix~\ref{Appendix_Discretisation}) captures the local and global behaviour of latent phase features, simulating the behaviour of a classical analytical operator. The kernel integral operator $\mathcal{K}$ is formally written as an integral over the domain $\mathcal{D}_\theta$, utilising a learnable kernel $\kappa^{(l)} \in C(\mathcal{D}_\theta \times \mathcal{D}_\theta; \mathbb{R}^{d_{l+1} \times d_l})$ and a Borel measure $\gamma$ on $\mathcal{D}_\theta$. Specifically, $\mathcal{K}$ is defined as $\left(\mathcal{K} v_l\right)(x)=\int_{\mathcal{D}_\theta} \kappa^{(l)}(x, y) v_l(y) d \gamma(y)$ for all $x \in \mathcal{D}_\theta$. On the graph, we are dealing with a discretised version of $\mathcal{K}$, where the kernel is based on the message-passing algorithm framework \cite{gilmer_neural_nodate}, employing average aggregation to update node features $h_j$ and obtain the operator solution $u\left(x_j\right)$, which reads:
\begin{equation}
\label{eq:DiscretisedKernelDefinition}
\begin{aligned}
    u\left(x_j\right)=\frac{1}{\left|\mathcal{N}_j\right|} \sum_{y \in \mathcal{N}_j} \kappa\left(x_j, y\right) v(y), \quad j=1, \ldots, J.
\end{aligned}
\end{equation}
This representation corresponds to the Monte-Carlo approximation of the integral. The multi-head attention layer in the proposed transformer resembles the integration kernel defined in~\eqref{eq:DiscretisedKernelDefinition} on the graph. We next demonstrate that the neural operator framework is a continuum generalisation of the graph transformer architecture.

\begin{proposition}
\label{main_proposition}
The residual block of the graph transformer layer, as proposed above, can be seen as a special case of the integration kernel of the neural operator.
\end{proposition}
\begin{proof}
Let us rewrite the concatenation as a full vector considering that different heads act on different parameters. This yields the following.
\begin{equation}
\label{eq:Proposition1}
    \hat{h}_i^{\ell}=\frac{O_h^{\ell}}{\left|\mathcal{N}_i\right|}\sum_{j\in\mathcal{N}_i}
    \begin{bmatrix}
    w_{i j}^{1, \ell} V^{1, \ell} & 0 & 0 \\
    0 & \ddots & 0 \\
    0 & 0 & w_{i j}^{H, \ell} V^{H, \ell} 
    \end{bmatrix}\cdot \begin{bmatrix}
    h_j^{\ell} \\
    \vdots \\
    h_j^{\ell} 
\end{bmatrix}
\end{equation}
Considering that the components are distinct entries in the diagonal matrix, we can rewrite this as:
$\hat{h}_i^{\ell}=\frac{O_h^{\ell}}{\left|\mathcal{N}_i\right|} \sum_{j \in \mathcal{N}_i}\left(w_{i j}^{1, \ell} V^{1, \ell}\oplus\cdots\oplus w_{i j}^{H, \ell} V^{H, \ell}\right)h^\ell_j$.
We then absorb $O_h^\ell$ into a block diagonal matrix to obtain:
$\hat{h}_i^{\ell}=\frac{1}{\left|\mathcal{N}_i\right|}\sum_{j \in \mathcal{N}_i}\left(w_{i j}^{1, \ell} \tilde{V}^{1, \ell}\oplus\cdots\oplus w_{i j}^{H, \ell} \tilde{V}^{H, \ell}\right)h^\ell_j$.
We use $\mathrm{ReLU}$ as the activation function $\sigma$. Denote $f^{n,l,i}_{j}:= \exp\left(\frac{Q^{n, \ell} h_i^{\ell} \cdot K^{n, \ell} h_j^{\ell}}{\sqrt{d_n}}\right)$, then we can write:
%
%
\begin{equation}
\label{eq:Proposition3}
\begin{aligned}
    h_i^{\ell+1} = \sigma\Bigg( W_1^{\ell}h_i^{\ell} &+ \frac{1}{|\mathcal{N}_i|}\sum_{j \in \mathcal{N}_i} \Bigg( 
    \frac{\exp\left(\frac{Q^{1, \ell} h_i^{\ell} \cdot K^{1, \ell} h_j^{\ell}}{\sqrt{d_1}}\right)}{\sum_{k\in\mathcal{N}_j} f^{1,l,k}_{j}} \tilde{\tilde{V}}^{1, \ell} \\
    &\oplus \cdots \oplus \frac{\exp\left(\frac{Q^{H, \ell} h_i^{\ell} \cdot K^{H, \ell} h_j^{\ell}}{\sqrt{d_H}}\right)}{\sum_{k\in\mathcal{N}_j}f^{H,l,k}_{j}} \tilde{\tilde{V}}^{H, \ell} \Bigg) h_j^{\ell} \Bigg) \\
    = \sigma\Bigg( W_1^{\ell}h_i^{\ell} &+ \frac{1}{|\mathcal{N}_i|}\sum_{j \in \mathcal{N}_i} \kappa h_j^{\ell} \Bigg) \\
\end{aligned}
\end{equation}
Defining $\kappa$ as the kernel and using the node features $h_j$ and $h_i$ as inputs, we obtain: 
$\mathcal{K} =\frac{1}{|\mathcal{N}_i|}\sum_{j\in\mathcal{N}_i}\kappa \left(h_i,h_j\right)h_j$, 
which can be considered as the Monte-Carlo approximation of the integral $u(x)=\int_{s(x)} \kappa(x, y) v(y) \mathrm{d} y, \quad \forall x \in\mathcal{D}_\theta$. Thus, the graph transformer is also applied to the class of neural operators.
\end{proof}

\subsection{Unboxing the \PhysGTN\ Architecture}\label{sec:methodology_arch}


\PhysGTN\ features three key modules: input parameter encoder (EncI), query and information fusion encoder (EncQ), and a decoder (Dec) for steady-state or time-dependent system output.

\faHandPointRight[regular] \textbf{\PhysGTN\ Encoder.}
As mentioned in Sec.~\ref{sec:methodology_graph}, the input graph $\mathcal{G}_\text{in}$ is constructed with input parameters and the corresponding positions as node features, using circular truncation for edge connectivity. $\mathcal{D}_\text{in}$ denote the spatial grid of the input parameters.
In EncI, the node features $\{h_i\}^L_{i=1}$ are embedded into latent space $\{\tilde{h_i}\}^L_{i=1}$ through a fully connected layer $\operatorname{FC}_\text{in}$, and the graph is updated as $\mathcal{\tilde{G}}_\text{in}$.
The data tuple $\{\mathcal{\tilde{G}}_\text{in}, \mathcal{D}_\text{in}\}$ is then fed into a stack of graph transformer blocks $\{\operatorname{GT}_{k}\}^{N_{\text{GT}}}_{k=1}$ discussed in Sec.~\ref{sec:methodology_graph_transformer}, which can be expressed as follows:
\begin{equation}
\begin{aligned}
    \mathcal{\tilde{G}}_\text{in} &= \operatorname{FC}_\text{in}(\mathcal{G}_\text{in}), \label{eq:encoder_i_1}\\
    \mathcal{G}_k &= \operatorname{GT}_{k}(\mathcal{G}_{k-1}, \mathcal{D}_\text{in}), \quad k = 1, \ldots, N_{\text{GT}}, 
\end{aligned}
\end{equation}
where $\mathcal{G}_0 = \mathcal{\tilde{G}}_\text{in}$. $\mathcal{G}_\text{EncI} = \mathcal{G}_{N_\text{GT}}$ is the output of EncI, with the node features $\{h_{\text{EncI},i}\}^{L}_{i=0}$, denoted as $\mathcal{H}_\text{EncI}$ for brevity.

The aim of EncQ is to encode the query location and integrate the query location embedding and the input parameter features.
Specifically, the query spatial grid $\mathcal{D}_\text{qry}$ ($L^{\prime}$-point discretised) is embedded by a Gaussian Fourier projection $\operatorname{GF}_\text{qry}$~\cite{tancik2020fourier} and then fed to a $\operatorname{MLP}_\text{qry}$ for the embedding of the query position $\{h_{\text{EncQ},i^\prime}\}^{L^\prime}_{i^\prime=0}$, denoted as $\mathcal{H}_\text{EncQ}$ for brevity. 
Subsequently, a cross-attention-based transformer block, namely, CrossFormer~\cite{li_transformer_2022}, is applied to integrate the $\mathcal{H}_\text{EncI}$ and $\mathcal{H}_\text{EncQ}$, which can be expressed:
\begin{subequations}
\begin{align}
    \mathcal{H}_\text{EncQ} &= \operatorname{MLP}_\text{qry}(\operatorname{GF}_\text{qry}(\mathcal{D}_\text{qry})), \label{eq:encoder_q_1}\\
    \mathcal{H}_\text{Enc} &= \operatorname{CF}(\mathcal{H}_\text{EncI}, \mathcal{H}_\text{EncQ}), \label{eq:encoder_q_2}
\end{align}
\end{subequations}
where $\operatorname{CF}$ denotes the CrossFormer and $\mathcal{H}_\text{Enc}$ are integrated latent features for the subsequent decoder.

Similarly to graph transformer blocks, the CrossFormer adopts a standard ``Attn-Norm-MLP-Norm'' with residual connection, where ``Attn'' is Galerkin-type cross-attention~\cite{cao2021choose, li_transformer_2022}, allowing for arbitrary query position. 
By introducing three sets of basis functions $\{k_l(\cdot), v_l(\cdot), q_l(\cdot)\}_{l=1}^d$, the cross-attention $\operatorname{CAttn}$ reads:
\begin{equation}
\begin{aligned}\label{eq:crossformer}
    \operatorname{CAttn}_j({h}_{\text{EncQ},i^\prime}) &= \sum_{l=1}^{d} \Bigg( \frac{\sum_{i=1}^{L} k_l({h}_{\text{EncI},i}) v_j({h}_{\text{EncI},i})}{L} \Bigg) \\
    &\quad \times q_l({h}_{\text{EncQ},i^\prime}).
\end{aligned}
\end{equation}

\faHandPointRight[regular] \textbf{\PhysGTN\ Decoder.} 
As discussed above, encoders lift the representation into a higher dimensional space, increasing the degrees of freedom and allowing for the interpretation of input parameters and query positions. The main task of the decoder is to project the latent feature back into the observable space.
For a steady-state system, such as Darcy Flow, an $\operatorname{MLP}_\text{out}$ is utilised to project the latent feature to observable space
\begin{equation}
\begin{aligned}\label{eq:decoder_sss}
    \hat{u} = \operatorname{MLP}_\text{out}(\mathcal{H}_\text{Enc} || \mathcal{D}_\text{qry}).
\end{aligned}
\end{equation}

However, for a time-dependent system, directly tracking all problem parameters and learning the whole temporal sequence raise complexity and lead to limited performance, especially problems with long temporal sequences, e.g. Diffusion Reaction. Consequently, inspired by~\cite{li_transformer_2022}, a recurrent-style propagator is adopted for time-dependent system decoding. The decision to utilise an RNN on top of the encoder features was driven by the requirement to capture temporal dependencies and propagate information across time steps, which is critical for the accurate simulation of PDEs over time. The latent state's role as input to the RNN over time is to maintain a continuous and coherent evolution of the system state, ensuring that predictions at each time step are informed by the accumulated knowledge of prior states. The encoder output is used as an initial state $\mathcal{H}_\text{Dec}^{0} = \mathcal{H}_\text{Enc}$, and fed into the propagator-MLP $\operatorname{MLP}_\text{prop}$ (with a residual connection) for recurrent running. After reaching a specific recurrent step, output-MLP $\operatorname{MLP}_\text{out}$ is used to project the latent feature to an observable space. 
The propagation in latent space from state $t$ to $t+\triangle t$, and the corresponding output projection, can be expressed as follows:
\begin{subequations}
\begin{align}
    \mathcal{H}_\text{Dec}^{t+\triangle t} &= \operatorname{MLP}_\text{prop}(\mathcal{H}_\text{Dec}^{t} || \mathcal{D}_\text{qry}) + \mathcal{H}_\text{Dec}^{t}, \label{eq:decoder_tds_1} \\
    \hat{u}^{t+\triangle t} &= \operatorname{MLP}_\text{out}(\mathcal{H}_\text{Dec}^{t+\triangle t} || \mathcal{D}_\text{qry}). \label{eq:decoder_tds_2}
\end{align}
\end{subequations}
The outputs of $\operatorname{MLP}_\text{out}$ are collected during the recurrent running $\hat{u}|_{t \in[\triangle t, T]} = [\hat{u}^{\triangle t}, \hat{u}^{2\triangle t}, \ldots,\hat{u}^{T}]$.

\faHandPointRight[regular]  \textbf{\PhysGTN\ Loss Function \& Discretisation Invariance.} 
We aim to learn the mapping between two infinite-dimensional spaces using a finite collection of dataset pairs. Our objective is to optimise the parameters $\mu = \{\mu_{\text{EncI}}, \mu_{\text{EncQ}}, \mu_{\text{Dec}} \}$ to obtain the best possible approximation $\tilde{\mathcal{S}}_{\mu_{opt}}\sim \mathcal{S}$, which represents a specific configuration of the operator map. We can solve this problem by minimising the empirical risk, considering an entire sampling set of $N$ independent and identically distributed (i.i.d.) datasets, which is given by:
\begin{equation}
\begin{aligned}\label{eq:loss}
\min _{\mu \in \mathbb{R}^p}\mathcal{L}_{\text {data }}\left(u, \mathcal{S}_\mu(\theta)\right) 
&= \min _{\mu \in \mathbb{R}^p}\mathbb{E}_{\theta \sim \nu}
\mathcal{L}(u, \mathcal{S}_\mu(\theta)) \\
&\approx \min _{\mu \in \mathbb{R}^p} \frac{1}{N} \sum_{i=1}^N
\mathcal{L}(u, \mathcal{S}_\mu(\theta)),
\end{aligned}
\end{equation}
where $\mathcal{L}_{\text{data}}$ is the empirical loss of distance function $\mathcal{L}$, which is either the MSE loss $\mathcal{L} = (u - \mathcal{S}_\mu(\theta))^2$ or the relative $L_2$ Norm Loss $\mathcal{L} = \frac{\|u - \mathcal{S}_\mu(\theta)\|_2}{\|u\|_2}$. The selection of loss function is based on the benchmark protocol; refer to Appendix~\ref{Appendix_Impl}.

The data pairs $\{\theta^{(n)},u^{(n)}\}_{n=1}^N$ are the discretisations of our continuous space. We assume that we only have access to the point-wise evaluation of finite points. Let $\mathcal{D}^{(n)}=\{x_{i}^{(n)}\}_{i=1}^{L_n} \subset \mathcal{D}$ be a $L_n$-point discretisation of the domain $\mathcal{D}$. Despite being trained on a $L_n$-point discretisation, the operator $\tilde{S}_\mu$ should be able to provide $u(x)$ for any arbitrary $x \in \mathcal{D}$ given an input $\theta \in \mathcal{P}$. In particular, the discretised architecture maps into the space $\mathcal{V}$ and not into a discretisation thereof. For additional details, see Appendix~\ref{Appendix_Discretisation}.

\section{Experimental Results}
In this section, we detail all the experiments carried out to validate our proposed \PhysGTN\ technique.

\subsection{Dataset \& Implementation Details}
\fcircle[fill=ocra]{2.5pt} \textbf{Dataset and PDEs.} We mainly utilised the datasets from PDEBench~\cite{takamoto2022pdebench} -- a wider range public benchmark for PDE-based simulation tasks, selecting Darcy Flow, Shallow Water, and Diffusion Reaction showcasing the stationary and time-dependent problems on a uniform grid. We also utilised the Airfoil dataset from~\cite{pfaff2021learning}, which models aerodynamics around an airfoil wing's cross-section, for experiments on irregular grids. Refer to Appendix~\ref{Appendix_Data} for details on the dataset and setting used.


\textbf{Darcy Flow.} We conducted experiments on the steady-state solution of 2D Darcy Flow over the uniform square. The setting used is detailed in Appendix~\ref{Appendix_Data}. \newline
%
%
%
\textbf{Shallow Water.} 
We performed the experiments on the time-dependent system of 2D Shallow Water equations, providing a 
model for analysing problems related to free-surface flows. Refer to Appendix~\ref{Appendix_Data} for details on our setting. \newline
%
%
%
%
%
%
\textbf{Diffusion Reaction.} 
The time-dependent system of 2D Diffusion-Reaction problems were included in the experiment. Our setting is detailed in Appendix~\ref{Appendix_Data}.
%
%
%
%
\newline
\textbf{Airfoil.} 
For non-uniform grid, we incorporated the airfoil problem based on 2D time-dependent compressible flow utilising datasets from~\cite{pfaff2021learning}. Refer to Appendix~\ref{Appendix_Data} for details on our setting. 

\fcircle[fill=ocra]{2.5pt} \textbf{Evaluation Protocol \& Implementation Details.}

We performed comparative experiments between our technique and several leading methods, including U-Net~\cite{ronneberger2015unet}\footnote{Using the PDEBench implementation \cite{takamoto2022pdebench}}, FNO~\cite{li2020fourier}, PINN~\cite{raissi_physics-informed_2019}, DeepONet~\cite{lu_deeponet_2021}, Geo-FNO~\cite{li2022fourier}, MAgNet~\cite{boussif2022magnet}, GNOT~\cite{hao2023gnot}, EAGLE~\cite{janny2023eagle} and OFormer~\cite{li_transformer_2022}.
For Darcy Flow, Shallow Water and Diffusion Reaction datasets, we report that the normalised root mean squared error (nRMSE) provides scale-independent information, following the PDEBench protocol~\cite{takamoto2022pdebench}. In addition, we report the results of RMSE in the Appendix~\ref{Appendix_Result}.
For Airfoil, we report the results using the relative $L_2$ error following the settings from~\cite{li_transformer_2022}.
All experiments were performed on a single NVIDIA A100 GPU with 80GB of memory, running under the same conditions for a fair comparison. For dynamical PDE settings, we trained \PhysGTN\ directly by unrolling for 90 timestamps. In contrast, UNet, FNO, GeoFNO, and MAgNet were trained using an auto-regressive approach. DeepONet was trained directly for the full timestamps, while OFormer was also trained by unrolling for the full timestamps.
More details about the model architecture, training setting, and hyper-parameter can be found in Appendix~\ref{Appendix_Impl}.

%

\begin{table*}[t!]
  \centering
  \caption{Numerical comparison of our approaches vs. existing techniques. The values reflect the nRMSE. The best-performing results are highlighted in \colorbox[HTML]{BBFFBB}{green}.}
  \resizebox{\textwidth}{!}{
    \begin{tabular}{ccccccccccc}
    \toprule
    \multicolumn{2}{c|}{\textsc{Dataset Setting}} & \multicolumn{9}{c}{Normalised RMSE (nRMSE)} \\
    \midrule
    Dataset & \multicolumn{1}{c|}{Param.} & GNOT & U-Net & FNO   & DeepONet & OFormer & GeoFNO & MAgNet & & \PhysGTN \\ \cmidrule{1-9}\cmidrule{11-11}
    Darcy Flow & \multicolumn{1}{c|}{$\beta = 0.01$} & -- & 1.10E+00 & 2.50E+00 & 3.25E-01 & \cellcolor[HTML]{BBFFBB}3.04E-01 & 1.03E+00 & 7.71E-02&  & \textbf{3.11E-01} \\
    Darcy Flow & \multicolumn{1}{c|}{$\beta = 0.1$} & -- & 1.80E-01 & 2.20E-01 & 1.67E-01 & 1.15E-01 & 3.13E-01 & 8.10E-02 & & \cellcolor[HTML]{BBFFBB}\textbf{7.65E-02} \\
    Darcy Flow & \multicolumn{1}{c|}{$\beta = 1.0$}& -- & 3.30E-02 & 6.40E-02 & 5.12E-02 & 2.05E-02 & 6.34E-02 & 1.03E-01 &  &\cellcolor[HTML]{BBFFBB}\textbf{1.40E-02} \\
    Darcy Flow & \multicolumn{1}{c|}{$\beta = 10.0$} & -- & 8.20E-03 & 1.20E-02 & 3.97E-02 & 6.34E-03 & 2.51E-02 & 1.62E-01 & &\cellcolor[HTML]{BBFFBB}\textbf{4.77E-03} \\
    Darcy Flow & \multicolumn{1}{c|}{$\beta = 100.0$} & -- & 4.40E-03 & 6.40E-03 & 3.64E-02 & 4.19E-03 & 2.04E-02 & 1.95E-01 & &\cellcolor[HTML]{BBFFBB}\textbf{3.46E-03} \\ \cmidrule{1-9}\cmidrule{11-11}
    Shallow Water & \multicolumn{1}{c|}{--}   & 4.16E-03 & 8.30E-02 & 4.40E-03 & 2.35E-03 & 2.90E-03 & 6.70E-03 & -- &   & \cellcolor[HTML]{BBFFBB}\textbf{2.04E-03} \\
    Diffusion Reaction & \multicolumn{1}{c|}{--}  & 8.22E-01 & 8.40E-01 & 1.20E-01 & 8.42E-01 & 3.28E+00 & 7.72E+00 & --   &  & \cellcolor[HTML]{BBFFBB}\textbf{9.02E-02} \\
    \bottomrule
    \end{tabular}
    }
  \label{tab:main_nrmse}%
\end{table*}%

\begin{figure}[t!]
  \centering
  \centerline{\includegraphics[width=0.5\textwidth]{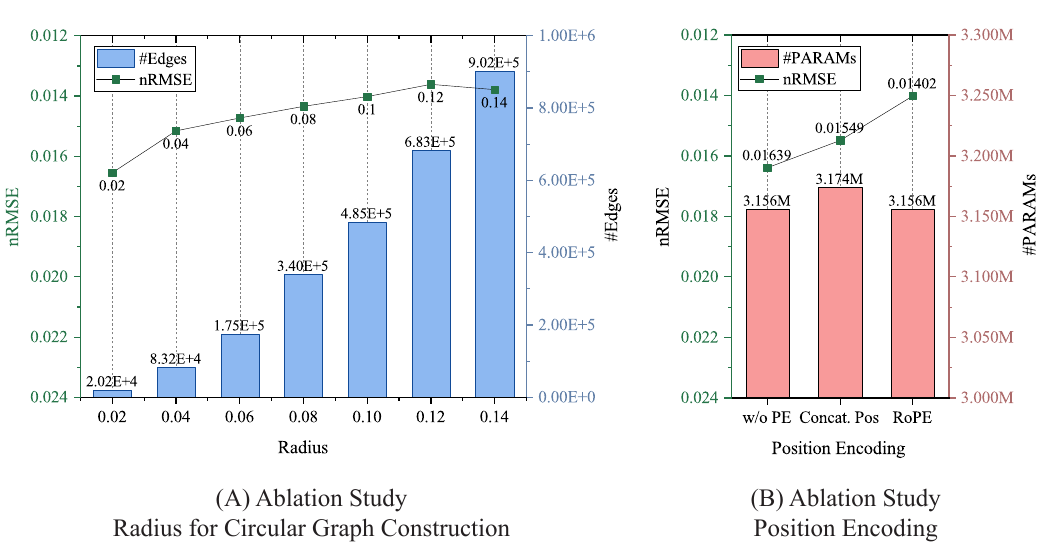}}
  \caption{
  Performance impacts of graph construction and position encoding on model accuracy. 
  (A) The trade-off between circular graph radius and nRMSE alongside edge count. 
  (B) Comparison of nRMSE across position encoding methods.}
  \label{fig:FIG_PLOT_ABL}
\end{figure}
\begin{figure}[t!]
  \centering
  \centerline{\includegraphics[width=0.5\textwidth]{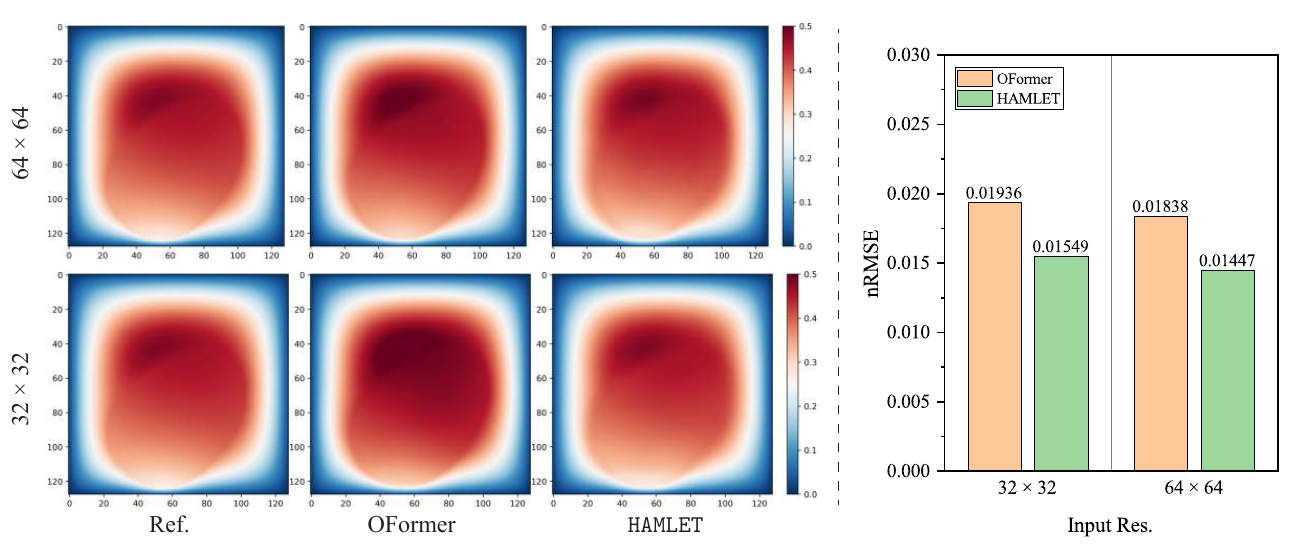}}
  \caption{
  Comparison of model predictions ($128\times128$) against a reference for two input resolutions ($32\times32$ and $64\times64$). 
  Left: Heatmap comparisons of the reference and predicted outcomes by OFormer and \PhysGTN\ models. 
  Right: Corresponding nRMSE values, demonstrating the superior accuracy of \PhysGTN.
  } \vspace{-0.2cm}
  \label{fig:FIG_PLOT_LOC}
\end{figure}
\begin{figure}[t!]
  \centering
  \centerline{\includegraphics[width=0.5\textwidth]{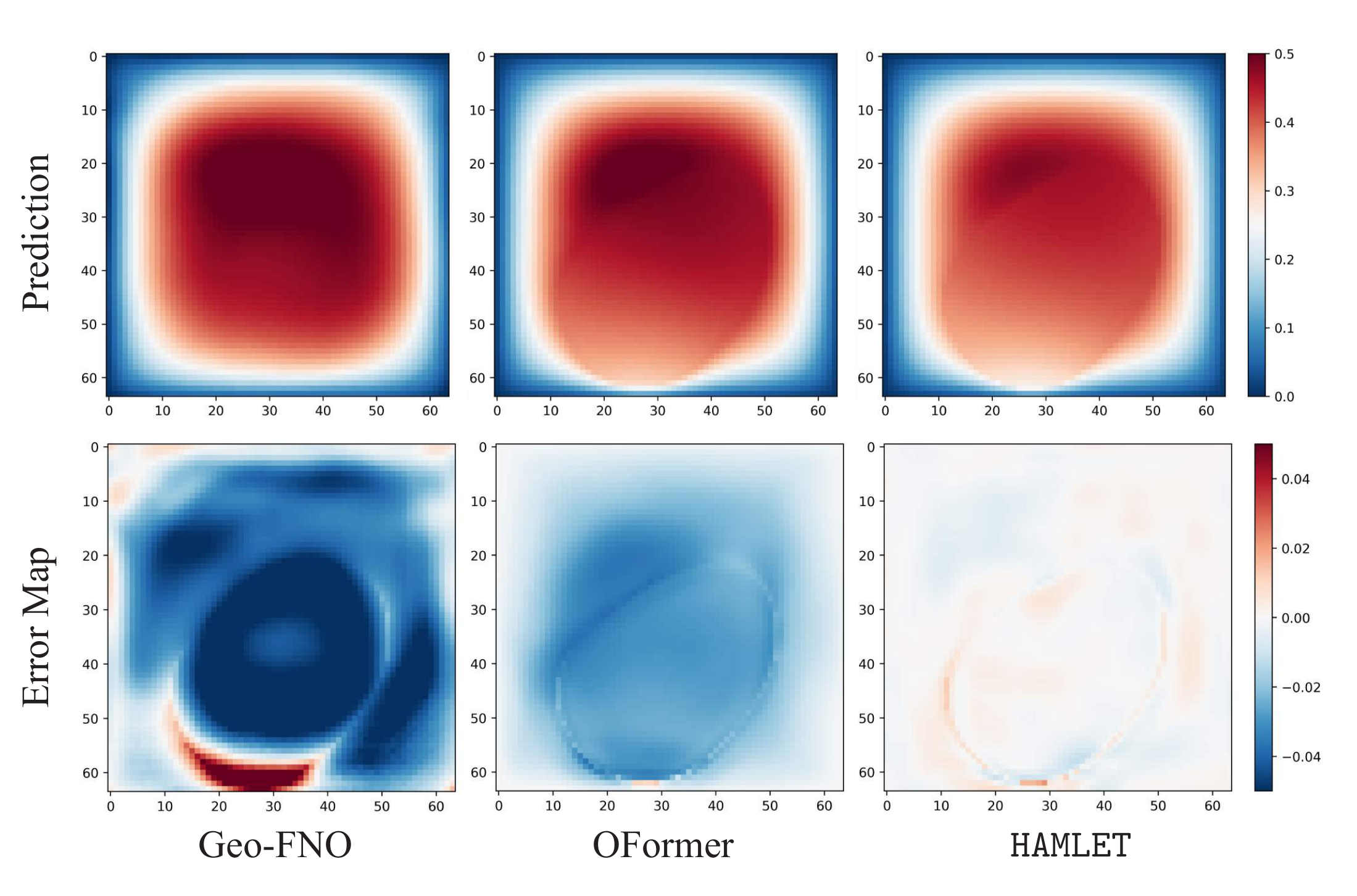}}
  \caption{Predictions and corresponding error maps for Geo-FNO, OFormer, and \PhysGTN\ models using Darcy Flow ($\beta=1.0$).} \vspace{-0.3cm}
  \label{fig:FIG_VIS_DarcyFlow_V2}
\end{figure}


\begin{figure*}[t!]
  \centering
  \centerline{\includegraphics[width=1\textwidth]{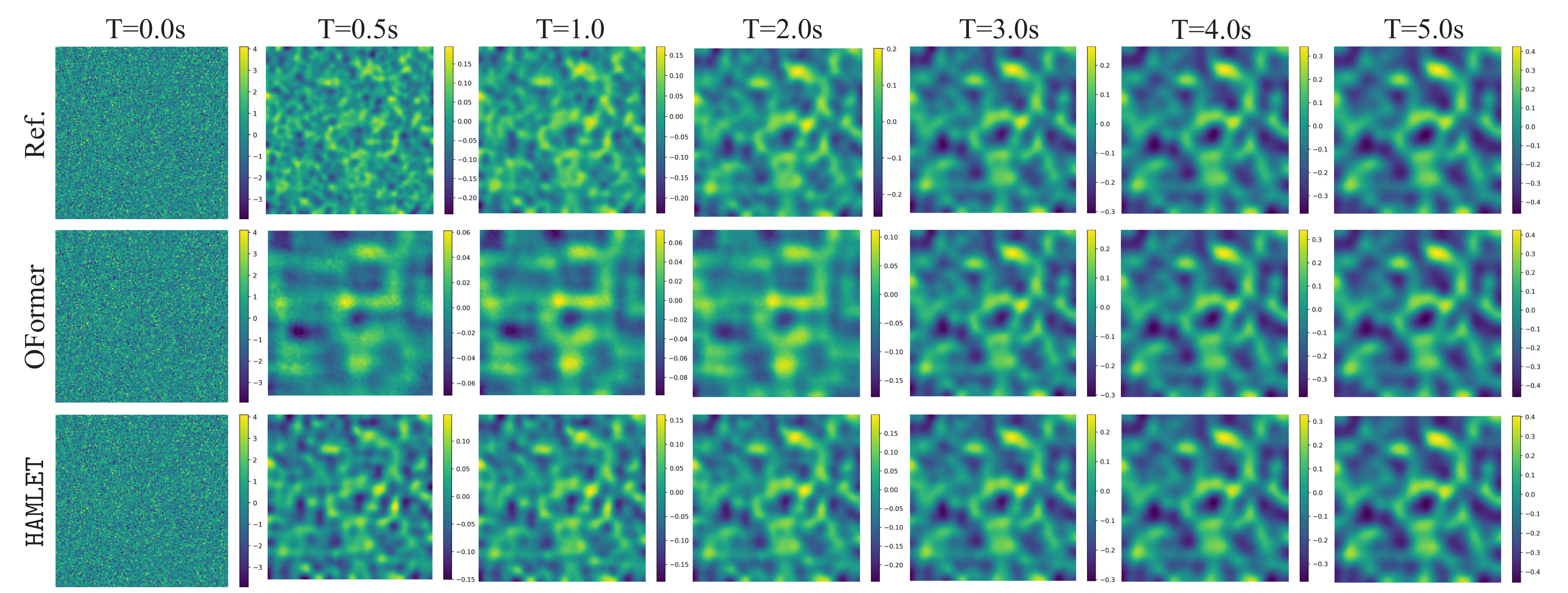}}
  \caption{
  Time evolution (activator) of the diffusion-reaction process as predicted by OFormer and \PhysGTN\ compared to the reference solution. 
  \PhysGTN 's predictions closely match the reference, showcasing its efficacy in dynamic PDE modelling.
  }
  \label{fig:FIG_VIS_DR2D_U}
\end{figure*}


\subsection{Results \& Discussion} 

\textbf{Numerical Comparison.} We carried out a thorough comparison of our 
framework against leading deep learning techniques for solving PDEs. The results are summarised in Table~\ref{tab:main_nrmse}\footnote{Note, that MAgNet was also trained on Shallow Water and Diffusion Reaction PDEs but the training failed to converge. Consequently, the results for MAgNet are not reported in Table~\ref{tab:main_nrmse}}.
Further supporting results are available in Appendix~\ref{Appendix_Result}. 
Our evaluation of the \PhysGTN\ approach in various datasets and parameter settings has demonstrated its robust performance. Notably, within the Darcy Flow dataset, \PhysGTN\ exhibited a commendable adaptability to changing parameters settings of $\beta$. The decrease in nRMSE as $\beta$ increased from 0.01 to 100.0 highlights \PhysGTN's capability to scale effectively with increased complexity or noise within the data. This trend suggests that \PhysGTN\ is not only robust, but also potentially more adept at handling complex simulations where other models may falter.

The generalisability of \PhysGTN\ was further underscored by its performance on the Shallow Water and Diffusion Reaction datasets. Its ability to maintain low nRMSE values across these diverse datasets suggests that \PhysGTN\ is not just tailored to a single type of physical simulation, but can be applied across various domains. This adaptability is critical for real-world applications, where conditions can be unpredictable and models must maintain a high degree of accuracy. Furthermore, there is a clear trend that as $\beta$ increases, the nRMSE decreases for \PhysGTN, suggesting that it scales well with complexity or noise within the data.

In addition, we conducted an in-depth comparison of a dynamic PDE modelling as predicted by OFormer and \PhysGTN\ against the reference solution. The results, illustrated in Figure \ref{fig:FIG_VIS_DR2D_U}, underscore the accuracy of \PhysGTN\'s predictions, which closely align with the reference, highlighting its effectiveness in dynamic PDE modelling. Our additional investigations reinforce \PhysGTN\'s state-of-the-art performance. Additional numerical experiments are also available under Appendix~\ref{Appendix_Result}.\\ \\
\textbf{Does the Graph Radius Affect Performance?} In Figure~\ref{fig:FIG_PLOT_ABL}-(A), the ablation study, using Darcy Flow, on the radius for the construction of the circular graph reveals a clear inverse relationship between the radius and the nRMSE. As the graph radius increases, there is an improvement in the nRMSE, indicating enhanced model performance. This trend continues until the radius reaches approximately 0.1 to 0.14, at which point the performance gains begin to plateau. This plateau suggests a convergence in performance improvement despite the graph becoming more complex with a greater number of edges.

\textbf{Performance Across Resolutions.} The visual representations in Figure~\ref{fig:FIG_PLOT_LOC} provide a comparison between the reference results and those predicted by OFormer (a non-graph transformer) and \PhysGTN\ models at different input resolutions. 
Qualitatively, the \PhysGTN\  produces heat maps that more closely resemble the reference, especially at the finer $64\times64$ resolution. 
This observation is quantitatively supported by the nRMSE values, where \PhysGTN\ consistently outperforms OFormer. At $32\times32$ resolution, \PhysGTN\ demonstrates a lower nRMSE to OFormer. The trend of superior performance continues at $64\times64$ resolution. These results suggest that \PhysGTN\ is better at capturing the intricacies of the dataset at different resolutions, which could be attributed to its graph-based perspective that may be more adept at encoding spatial relationships within the data.

\textbf{\PhysGTN\ -- Graph or non-Graph, that is the question.} 
We further support our method through Figure~\ref{fig:FIG_VIS_DarcyFlow_V2}.
This shows that a comparison with Geo-FNO underscores the effectiveness of \PhysGTN\, which is also graph-based but exhibits the lowest error among models. This not only highlights the utility of graph-based approaches in capturing the complexity of the data, but also shows the specific strengths of \PhysGTN\ in minimising prediction errors. 
Moreover, Table~\ref{tab:exp_airfoil} presents a concise quantitative analysis of the performance of \PhysGTN\ and OFormer on the Airfoil dataset, characterised by its non-uniform grid structure. The superior performance of \PhysGTN\ in this context not only confirms the adaptability of graph-based approaches to nonuniform grids, but also underscores the potential advantages of such methods over their transformer counterparts in capturing complex spatial relationships inherent in computational fluid dynamics datasets such as the Airfoil dataset.


\begin{table}[t!]
  \centering
  \caption{Ablation studies for graph constructions on Darcy Flow 2D ($\beta=1.0$) in $64 \times 64$ grid.}
  \resizebox{0.49\textwidth}{!}{%
    \begin{tabular}{ccccccrc}
    \toprule
    \multirow{2}[4]{*}{} &       & \multicolumn{4}{c}{KNN}       &       & \multirow{2}[4]{*}{Circular} \\
\cmidrule{3-6}          &       & $k=21$  & $k=51$  & $k=101$ & $k=151$ &       &  \\
\cmidrule{1-1}\cmidrule{3-6}\cmidrule{8-8}    nRMSE &       & 0.02018  & 0.01996  & 0.02054  & 0.02094  &       & \textbf{0.01402 } \\
    \bottomrule
    \end{tabular}%
    }
  \label{tab:abl_graph}%
\end{table}%

%
\begin{table}[t!]
  \centering
  \caption{Quantitative Results for models trained on different dataset sizes.}
    \resizebox{0.49\textwidth}{!}{%
    \begin{tabular}{ccccc}
    \toprule
    \multicolumn{5}{c}{Darcy Flow ($\beta=1.0$), nRMSE} \\
    \midrule
    \#Training Data & 9K    & 5K    & 2K    & 1K \\
    \midrule
    OFormer & 2.048E-02 & 2.093E-02 & 2.674E-02 & 3.321E-02 \\
    \PhysGTN & 1.402E-02 & 1.642E-02 & 2.211E-02 & 2.779E-02 \\
    \midrule
    \multicolumn{5}{c}{Shallow Water, nRMSE} \\
    \midrule
    \#Training Data & 900   & 500   & 200   & 100 \\
    \midrule
    OFormer & 2.900E-03 & 1.190E-02 & 2.310E-02 & 2.910E-02 \\
    \PhysGTN & 2.044E-03 & 2.320E-03 & 3.255E-03 & 4.746E-03 \\
    \bottomrule
    \end{tabular}%
    }
  \label{tab:exp_dsize}%
\end{table}%


%

\textbf{Can Position Encoding Transform Model Performance?} The results of Figure~\ref{fig:FIG_PLOT_ABL}-(B) demonstrate the impact of different position encoding strategies on model performance. The Rotary Position Embedding (RoPE) outperforms other methods, indicating that incorporating relational information between positions can significantly enhance the model's understanding of the spatial structure within the data. Directly concatenating the coordinates to the input features offers a performance benefit over not using any position encoding, but it is still inferior to the relational encoding offered by RoPE. This suggests that while the model can derive some spatial context from direct coordinate information, the rotational invariance and relative positional information encoded by RoPE provide a more nuanced and effective representation. This emphasises the value of sophisticated position encoding mechanisms of our graph framework.

\textbf{Graph Construction Ablation.} Our ablation study for graph construction on the Darcy flow 2D dataset ($\beta$=1.0), indicates a clear advantage of our circular strategy over the traditional k-nearest neighbours (KNN) method. By defining neighbours within a variable radius, our approach dynamically captures the spatial context of the dataset, resulting in a significant performance leap with the lowest nRMSE compared to KNN's. This single-radius, context-aware strategy outperforms the fixed-edge KNN method, confirming the importance of spatial relationships in accurately modelling complex physical phenomena.

\begin{table}[t!]
  \centering
  \caption{Quantitative Results on Airfoil Dataset}
    \setlength{\tabcolsep}{10mm}{%
    \resizebox{0.35\textwidth}{!}{
    \begin{tabular}{cc}
    \toprule
    \multicolumn{2}{c}{Airfoil, Relative $L_2$ Norm} \\
    \midrule
    OFormer & 3.486E-02 \\
    GNOT & 4.310E-02 \\
    EAGLE & 1.192E-01 \\
    \PhysGTN & 3.030E-02 \\
    \bottomrule
    \end{tabular} 
    }
    }
  \label{tab:exp_airfoil}%
\end{table}%

\textbf{Data Size Performance Impact.} We examine the effect of dataset size on model performance. From Table~\ref{tab:exp_dsize}, we can observe that both OFormer and our proposed \PhysGTN\ model exhibit a decline in performance as the amount of training data decreases. This trend is evident in both the Darcy Flow and Shallow Water datasets, with nRMSE values rising as the dataset size reduces from 9K to 1K and from 900 to 100 respectively. Despite this common trend, \PhysGTN\ consistently outperforms OFormer in all datasets, demonstrating its superior ability to generalise from limited data. Particularly in the most constrained scenarios—1K for Darcy Flow and 100 for Shallow Water— \PhysGTN\  demonstrates greater resilience to data scarcity. \PhysGTN\ demonstrates improved performance under data-scarce conditions because it utilises a graph-based approach, which inherently requires fewer data points to model complex relationships effectively compared to traditional methods. Our graph perspective allows the neural operator to infer and generalise from less training data by leveraging the intrinsic geometrical and topological information contained within the graphs, which represent physical domains or conditions in PDEs. This approach enables more efficient learning, particularly advantageous when dealing with limited datasets commonly encountered in PDE scenarios.

\section{Conclusion} 
Our \PhysGTN\ model introduces an advancement in the neural network-based resolution of PDEs, offering a flexible and robust solution adaptable to various geometries and conditions. \PhysGTN\ uses of graph transformers and modular input encoders establishes new benchmarks in the field, particularly in complex scenarios with limited data availability. A limitation is the graph construction time, which is a common aspect in graph-based approaches. However, we view this not as a major constraint but as an inherent step that enables our model's robust performance. The construction of the graph is a crucial phase in which \PhysGTN\ captures the complex dependencies within the data. We believe that the benefits gained in accuracy and adaptability to diverse PDEs far outweigh the computational time required during this stage.
Our extensive testing confirms that \PhysGTN\ not only meets but often exceeds the performance of existing techniques. 
Future work includes integration of Lie-symmetry preservation and augmentation. But also includes extending \PhysGTN\ to handle higher-dimensional PDEs such as 3D problems. This will involve a dedicated exploration to refine the model architecture and compare it with existing approaches, similar to the methodology used by Li et al. for large-scale 3D PDEs \cite{li2023geometryinformed}.

\Jiahao{
}

\section*{Acknowledgments}
AB was supported in part by the Akamai Presidential Fellowship and the Hans-Messer Foundation.
JH and GY were supported in part by the ERC IMI (101005122), the H2020 (952172), the MRC (MC/ PC/21013), the Royal Society (IEC\ NSFC211235), the NVIDIA Academic Hardware Grant Program, the SABER project supported by Boehringer Ingelheim Ltd, Wellcome Leap Dynamic Resilience, and the UKRI Future Leaders Fellowship (MR/V023799/1).
CBS acknowledges support from the Philip Leverhulme Prize, the Royal Society Wolfson Fellowship, the EPSRC advanced career fellowship EP/V029428/1, EPSRC grants EP/S026045/1 and EP/T003553/1, EP/N014588/1, EP/T017961/1,
the Wellcome Innovator Awards 215733/Z/19/Z and 221633/Z/20/Z, CCMI and the Alan Turing Institute.
AAR gratefully acknowledges funding from the Cambridge Centre for Data-Driven Discovery and Accelerate Programme for Scientific Discovery, made possible by a donation from Schmidt Futures, ESPRC Digital Core Capability Award, and CMIH and CCIMI, University of Cambridge.

\section*{Impact Statement}
The impact of our work is observed primarily in the domain of solving partial differential equations (PDEs) using neural networks. Although our work does not directly address specific ethical concerns, such as those related to large language models or image generation, advanced models like \PhysGTN\ can influence a wide range of applications, from climate modelling to biomedical simulations. Therefore, we encourage continued reflection on how advancements in machine learning may shape and be shaped by societal and ethical considerations.
\bibliography{ICML24}
\bibliographystyle{icml2024}

\newpage
\appendix
\onecolumn

\section*{Appendix}
This document builds upon the practical aspects and visual results detailed in the main paper, providing additional insights into our methodology and experimental results. 

\section{Further Numerical \& Visual Results.}~\label{Appendix_Result}
In this section, we extend the numerical and visual comparison of Table \ref{tab:main_nrmse} and Figure \ref{fig:FIG_VIS_DR2D_U} of the main paper.  The RMSE results presented in Table~\ref{tab:main_rmse} illustrate the robustness of the \PhysGTN\ approach across various datasets and complexity levels, consistently showcasing competitive or superior performance compared to existing models.  Particularly in the DarcyFlow datasets with extreme $\beta$ values and the physically intricate Shallow Water and Diffusion Reaction datasets, \PhysGTN 's graph-based methodology excels in capturing complex spatial relationships, often outperforming existing models including a nongraph transformer-based technique.\footnote{The use of GeoFNO for the Darcy Flow dataset does not imply variation in geometry within the Darcy equation itself, as the Darcy Flow problem typically involves a fixed geometry. Instead, the application of GeoFNO likely aims to leverage its learning capacity on complex geometries and could be used to assess the robustness of the model even when the underlying PDE does not exhibit geometric variability. The model's application in this context might be more about demonstrating its potential in handling geometric complexities should they arise in other PDEs, rather than indicating a varying geometry within the Darcy Flow problem itself.} \vspace{-0.3cm}

\begin{table*}[h!]
  \centering
  \caption{Numerical comparison of our approaches vs. existing techniques. The values reflect the RMSE. The best-performing results are highlighted in \colorbox[HTML]{BBFFBB}{green}. }
  \resizebox{\textwidth}{!}{
    \begin{tabular}{cccccccccc}
    \toprule
    \multicolumn{2}{c|}{\textsc{Dataset Setting}} & \multicolumn{8}{c}{RMSE} \\
    \midrule
    Dataset & \multicolumn{1}{c|}{Param.} & U-Net & FNO   & DeepONet & OFormer & GeoFNO & MAgNet &  & \PhysGTN \\ \cmidrule{1-8}\cmidrule{10-10}
    DarcyFlow & \multicolumn{1}{c|}{$\beta = 0.01$}  & 4.00E-03 & 8.00E-03 & 3.31E-03 & \cellcolor[HTML]{BBFFBB}\textbf{2.21E-03} & 2.70E-03 & 8.07E-03 &  & 2.45E-03 \\
    DarcyFlow & \multicolumn{1}{c|}{$\beta = 0.1$}   & 4.80E-03 & 6.20E-03 & 4.88E-03 & 2.55E-03 & 4.15E-03 & 1.05E-02 &  & \cellcolor[HTML]{BBFFBB}\textbf{2.60E-03} \\
    DarcyFlow & \multicolumn{1}{c|}{$\beta = 1.0$}   & 6.40E-03 & 1.20E-02 & 9.65E-03 & 3.00E-03 & 6.20E-03 & 2.90E-02 &  & \cellcolor[HTML]{BBFFBB}\textbf{2.74E-03} \\
    DarcyFlow & \multicolumn{1}{c|}{$\beta = 10.0$}  & 1.40E-02 & 2.10E-02 & 6.79E-02 & 7.32E-03 & 2.08E-02 & 2.18E-01 &  & \cellcolor[HTML]{BBFFBB}\textbf{5.51E-03} \\
    DarcyFlow & \multicolumn{1}{c|}{$\beta = 100.0$} & 7.30E-02 & 1.10E-01 & 6.21E-01 & 4.91E-02 & 1.65E-01 & 2.11E+00 &  & \cellcolor[HTML]{BBFFBB}\textbf{3.37E-02} \\
    \cmidrule{1-8}\cmidrule{10-10}
    Shallow Water & \multicolumn{1}{c|}{--}    & 8.60E-02 & 4.50E-03 & 2.44E-03 & 3.10E-03 & 7.30E-03 & --    &  & \cellcolor[HTML]{BBFFBB}\textbf{2.13E-03} \\
    Diffusion Reaction & \multicolumn{1}{c|}{--}    & 6.10E-02 & 8.10E-03 & 6.46E-02 & 1.26E-02 & 5.20E-02 & --    &  & \cellcolor[HTML]{BBFFBB}\textbf{5.48E-03} \\
    \bottomrule
    \end{tabular}%
    }
  \label{tab:main_rmse}%
\end{table*}%

Figure~\ref{fig:FIG_VIS_SW2D} presents a comparative visualisation of the time evolution predictions for the Shallow Water problem made by OFormer, a transformer-based model, and \PhysGTN\, our graph-based transformer approach, against the reference solution. 
Across the sequence from T=0s to T=1s, \PhysGTN's predictions are more closely aligned with the reference, particularly at later time steps where the complexity of the system behaviour increases. This visual alignment is quantitatively supported by our previous results, which showed \PhysGTN achieving lower error metrics compared to OFormer. The comparison emphasises \PhysGTN 's superior performance in capturing the dynamic nature of the system, likely due to its graph-based structure, which enables a more nuanced interpretation of spatial-temporal data.

Figure~\ref{fig:FIG_VIS_DR2D_V} presents complementary results to Figure~\ref{fig:FIG_VIS_DR2D_U}, offering a comparative visualisation of the time evolution predictions (inhibitor) for the Diffusion Reaction process made by OFormer and our \PhysGTN\ in comparison to the reference solution. We want to note that the unique characteristics of the Diffusion Reaction system, particularly the non-linear reaction terms, lead to complex patterns from initial conditions. \PhysGTN\ architecture is better suited for capturing the intricate dynamics, due to our graph perspective, of this system early on due to its ability to model complex spatial relationships and non-linear interactions more effectively than OFormer, especially under conditions of rapid change and high-frequency information present at the initial time steps. This advantage is particularly evident in the early stages, where the intricate interplay of diffusion and reaction is most challenging to model accurately. Therefore, we intentionally chose to showcase these results to highlight the advantages of our approach.

\begin{figure*}[h!]
  \centering
  \centerline{\includegraphics[width=1\textwidth]{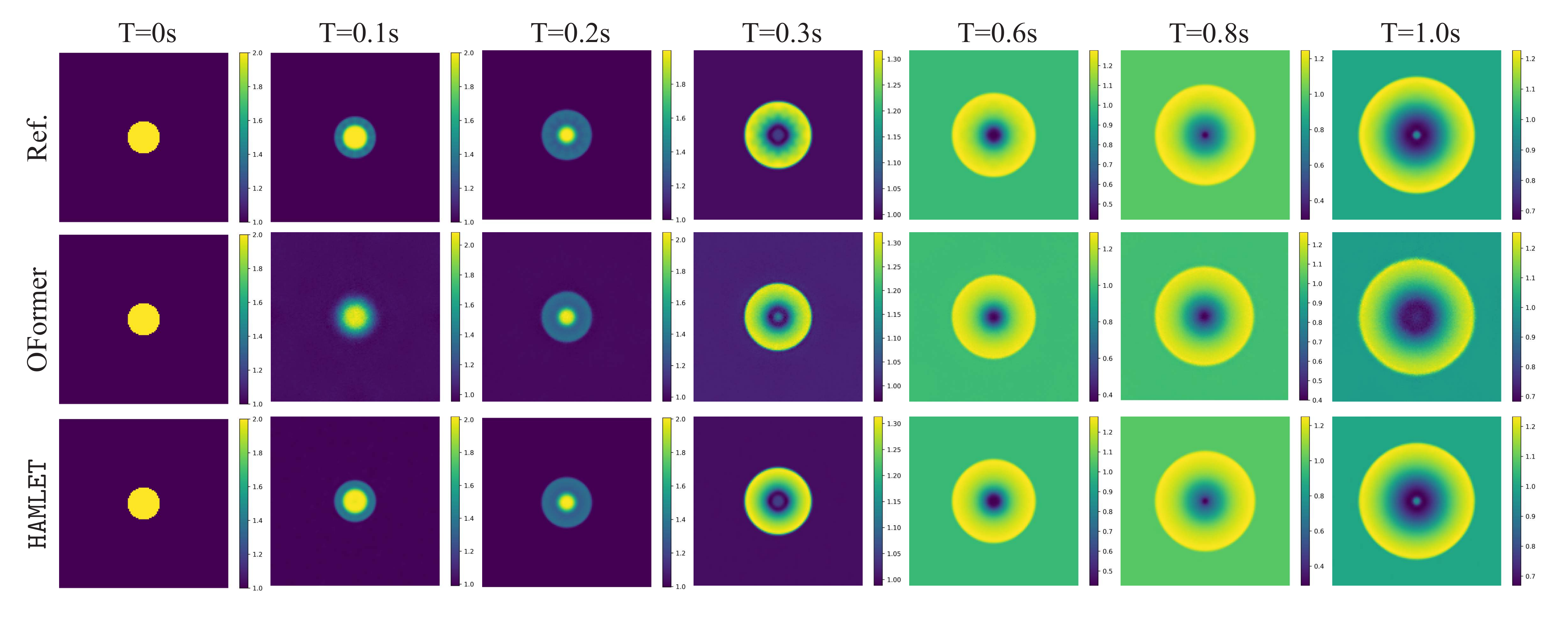}}
  \caption{Time-lapse predictions of a dynamic process by OFormer and \PhysGTN\ compared with the reference solution using Shallow Water 2D. \PhysGTN 's graph-based approach yields closer alignment to the reference.}
  \label{fig:FIG_VIS_SW2D}
\end{figure*}

\begin{figure*}[t!]
  \centering
  \centerline{\includegraphics[width=1\textwidth]{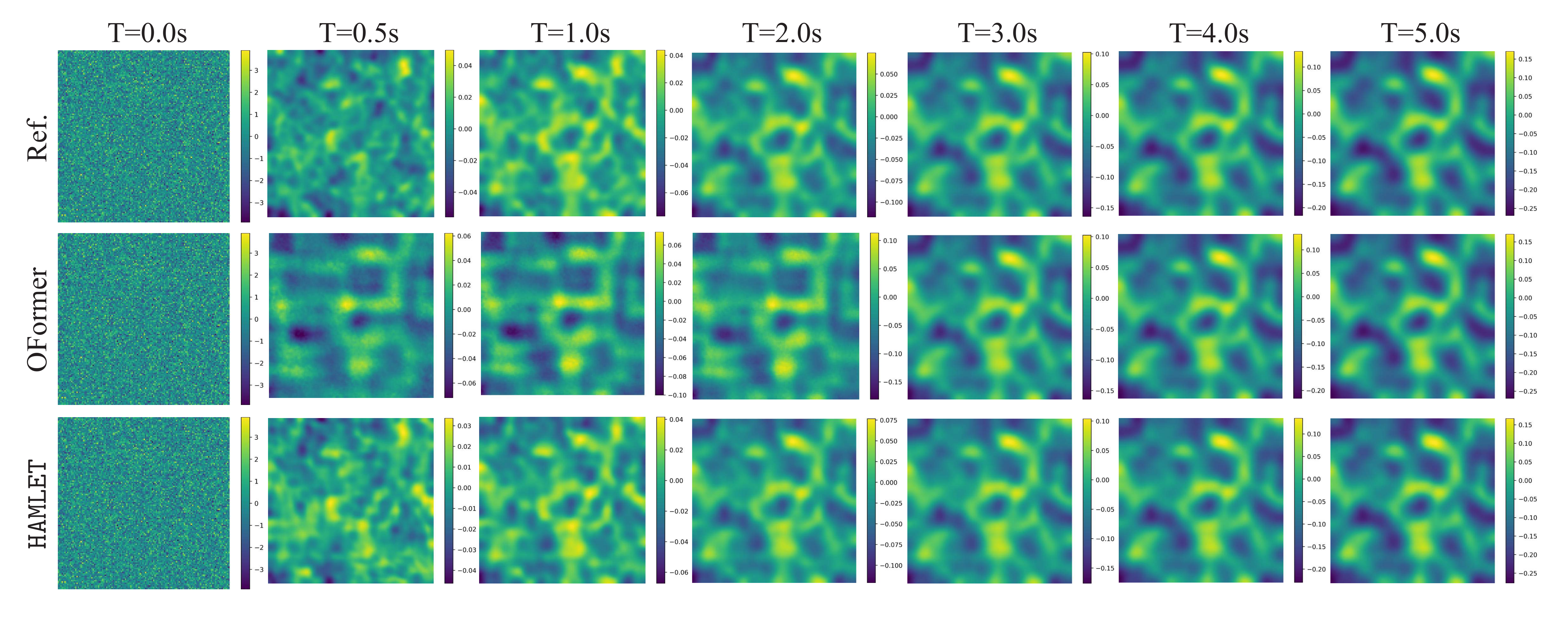}}
  \caption{
  Time evolution (inhibitor) of the diffusion-reaction process as predicted by OFormer and \PhysGTN\ compared with the reference solution. 
  \PhysGTN's predictions closely align with the reference, indicating its strong performance in modelling dynamic PDE systems.
  }
  \label{fig:FIG_VIS_DR2D_V}
\end{figure*}

\section{Discretisation Invariance \& Further Definitions.}~\label{Appendix_Discretisation}

\fcircle[fill=wine]{2.5pt} \textbf{Discretisation Invariance.} Suppose $\mathcal{P}$ is a Banach space of $\mathbb{R}^s$-valued functions on the domain $\mathcal{D} \subset \mathbb{R}^n$. Let $\mathcal{S}: \mathcal{D} \rightarrow \mathcal{V}$ be an operator, $\mathcal{D}_L$ be an $L$-point discretisation of $\mathcal{D}$, and $\tilde{\mathcal{S}}_\mu$ the approximate solution operator. For any compact $K \subset \mathcal{P}$, one can define the maximum offset of the operators following the ideas of Kovachki et al.~\cite{kovachki_neural_2022}:
\begin{equation}
\label{eq:MaximumOffset}
R_K\left(\mathcal{S}, \tilde{\mathcal{S}}_\mu, \mathcal{D}_L\right)=\sup _{\theta \in K}\left\|\hat{\mathcal{S}}\left(\mathcal{D}_L,\left.\theta\right|_{\mathcal{D}}\right)-\mathcal{S}(\theta)\right\|_{\mathcal{V}} 
\end{equation}
Given a set of discretisations $\{\mathcal{D}^{(i)}\}_{i=1}^{\infty}$ with the condition $\mathcal{D}^{(i)}\subset \mathcal{D}^{(j)}$ for $j>i$ of the domain $\mathcal{D} \subset \mathbb{R}^n$ such that $\mathcal{D} \subset\cup_{i=1}^{\infty} \mathcal{D}^{(i)}$, we say $\tilde{\mathcal{S}}_\mu$ is discretisation-invariant if the offset tends to zero when evaluating and training the neural operator $\tilde{\mathcal{S}}^{(i)}$ on other $N_i$ discretisations, for any $\theta \in \mathcal{P}$ and any compact set $K \subset \mathcal{P}$:
\begin{equation}
\label{eq:LimitMaximumOffset}
\lim _{N \rightarrow \infty} R_K\left(\mathcal{S}(\cdot, \theta), \tilde{\mathcal{S}}^{(N)}(\cdot, \cdot, \theta), \mathcal{D}^{(N)}\right)=0  
\end{equation}
To show that a neural network architecture is numerically discretisation-invariant, we must show that the approximation error is approximately constant as we refine the discretisation.

\fcircle[fill=wine]{2.5pt} \textbf{Neural Operator.}\begin{definition}[Neural operator $\tilde{\mathcal{S}}_\mu$]
Define the neural operator
\begin{equation}
\label{eq:NeuralOperatorDefinition}
\tilde{\mathcal{S}}_\mu:=\mathcal{Q} \circ\left(W_L+\mathcal{K}_L+b_L\right) \circ \cdots \circ \sigma\left(W_1+\mathcal{K}_1+b_1\right) \circ \mathcal{P}
\end{equation}
where $\mathcal{P}: \mathbb{R}^{d_a} \rightarrow \mathbb{R}^{d_1}, \mathcal{Q}: \mathbb{R}^{d_L} \rightarrow \mathbb{R}^{d_u}$ are the pointwise neural networks that encode the lower dimension function into higher dimensional space and vice versa. The model stack $L$ layers of $\sigma\left(W_l+\mathcal{K}_l+b_l\right)$ where $W_l \in \mathbb{R}^{d_{l+1} \times d_l}$ are pointwise linear operators (matrices), $\mathcal{K}_l:\{\mathcal{D} \rightarrow$ $\left.\mathbb{R}^{d_l}\right\} \rightarrow\left\{\mathcal{U} \rightarrow \mathbb{R}^{d_{l+1}}\right\}$ are integral kernel operators, $b_l: D \rightarrow \mathbb{R}^{d_{l+1}}$ are the bias terms made of a linear layer, and $\sigma$ are fixed activation functions. The parameters $\mu$ consist of all the parameters in $\mathcal{P}, \mathcal{Q}, W_l, \mathcal{K}_l, b_l$
\end{definition}

\section{Further Details on the Datasets \& PDEs.}\label{Appendix_Data}

\fcircle[fill=wine]{2.5pt} \textbf{Darcy Flow.}
The 2D Darcy Flow equation is defined by:
\begin{equation}
\label{eq:DarcyFlowEquation}
\begin{aligned}
    -\nabla(a(x, y) \nabla u(x, y)) & =f(x, y), & (x, y) \in \Omega, \\
    u(x, y) & =0, &(x, y) \in \partial\Omega,
\end{aligned}
\end{equation}
where $a(x, y)$ and $u(x, y)$ are the diffusion coefficient and the solution respectively, which are defined on a square domain $\Omega = (0, 1)^2$. The force term $f(x)$ is set to a constant value $\beta$, which affects the scale of the solution $u(x)$.

We conducted experiments on the steady-state solution of 2D Darcy Flow over the uniform square. 
The objective solution operator $\mathcal{S}$ is defined as follows:
\begin{equation}
\label{eq:DarcyFlowOp}
\begin{aligned}
    \mathcal{S}: a \mapsto u, \quad (x, y) \in \Omega,
\end{aligned}
\end{equation}
where $a(x, y)$ and $u(x, y)$ are the diffusion coefficient and the solution respectively, which are defined on a square domain $\Omega = (0, 1)^2$.
Following the PDEBench~\cite{takamoto2022pdebench} protocol, five subsets with constant value force term $\beta=0.01, 0.1, 1.0, 10.0, 100.0$ were used in the experiment section, each of which contains 9000/1000 samples for training/testing sets. All data samples are officially discretised and $\times 2$ subsampled in a spatial resolution of $64 \times 64$.

\fcircle[fill=wine]{2.5pt} \textbf{Shallow-Water.} 
The 2D Shallow Water equation is written as follows:
\begin{equation}
\label{eq:ShallowWaterEquation}
\begin{aligned} 
    \partial_t h+\partial_x h u+\partial_y h v &=0, \\ 
    \partial_t h u+\partial_x\left(u^2 h+\frac{1}{2} g_r h^2\right) &=-g_r h \partial_x b, \\ 
    \partial_t h v+\partial_y\left(v^2 h+\frac{1}{2} g_r h^2\right) &=-g_r h \partial_y b,
\end{aligned}
\end{equation}
where $u := u(x, y, t)$ and $v := v(x, y, t)$ describes the velocities in horizontal and vertical direction and $h := h(x, y, t)$ describes the water depth. $b := b(x, y)$ and $g_r$ denote spatially variable bathsymetry and gravitational acceleration, respectively. 
This specific dataset describes a 2D radial dam break scenario on a square domain $\Omega=[-2.5, 2.5]^2$ and time range $t \in [0, 1]$.
with the following initialisation setting:
\begin{equation}
\label{eq:ShallowWaterInit}
\begin{aligned} 
    h(t=0, x, y)= \begin{cases}2.0, & \text { for } r<\sqrt{x^2+y^2} \\ 1.0, & \text { for } r \geq \sqrt{x^2+y^2}\end{cases}
\end{aligned}
\end{equation}
where the radius $r \sim \mathcal{D}(0.3,0.7)$.

We conducted experiments on the time-dependent system of 2D Shallow Water equations, providing an appropriate model for analysing problems related to free-surface flows. 
The objective solution operator $\mathcal{S}$ is defined as follows:
\begin{equation}
\label{eq:ShallowWaterOperator}
\begin{aligned} 
    \mathcal{S}: h|_{t \in[0, t^\prime]} \mapsto h|_{t \in (t^\prime, T]}, \quad (x, y) \in \Omega,
\end{aligned}
\end{equation}
where $t^\prime = 0.009s$ and $T = 1.000s$, $\Omega=[-2.5, 2.5]^2$. $h := h(x, y, t)$ describes the depth of the water. 

Each sample in the dataset is officially discretised into a spatial resolution of $128 \times 128$ and a time resolution of $101$ respectively, where the first 10 time points were used as input and the rest of 91 as target. Following the PDEBench~\cite{takamoto2022pdebench} protocol, the dataset contain 900/100 for training/testing.

\fcircle[fill=wine]{2.5pt} \textbf{Diffusion Reaction.} 
The Diffusion Reaction equations are rewritten as follows:
\begin{equation}
\label{eq:DiffusionReactionEquation}
\begin{aligned} 
    \partial_t u &= D_u \partial_{x x} u+D_u \partial_{y y} u+R_u, \\
    \partial_t v &= D_v \partial_{x x} v+D_v \partial_{y y} v+R_v,
\end{aligned}
\end{equation}
where the activator $u = u(x, y, t)$ and the inhibitor $v = v(x, y, t)$ are two non-linearly coupled variables. The diffusion coefficients for the activator and inhibitor are $D_u = 1 \times 10^{-3}$ and $D_v = 5 \times 10^{-3}$, respectively.
The activator and inhibitor reaction functions are defined as follows:
\begin{equation}
\label{eq:DiffusionReactionEquation2}
\begin{aligned} 
    R_u(u, v)=u-u^3-k-v, \qquad
    R_v(u, v)=u-v,
\end{aligned}
\end{equation}
where $k = 5 \times 10^{-3}$.
The simulation domain includes $\Omega=[-1, 1]^2$ and $t \in [0, 5]$.

The objective solution operator $\mathcal{S}$ is defined as follows:
\begin{equation}
\label{eq:DiffusionReactionOperator}
\begin{aligned} 
    \mathcal{S}: \{u, v\}|_{t \in[0, t^\prime]} \mapsto \{u, v\}|_{t \in (t^\prime, T]}, \quad (x, y) \in \Omega,
\end{aligned}
\end{equation}
where $t^\prime = 0.045s$, $T = 5.000s$ and $\Omega=[-1, 1]^2$. $u := u(x, y, t)$ and $v := v(x, y, t)$ is the activator and inhibitor. 
Similarly to the Shallow Water dataset, each sample is officially discretised and subsampled at a spatial resolution of $128 \times 128$ and a time resolution of $101$ (10 for input and 91 for target).
Following the PDEBench~\cite{takamoto2022pdebench} protocol, the dataset contain 900/100 for training/testing.

\fcircle[fill=wine]{2.5pt} \textbf{Airfoil.} 
The 2D time-dependent compressible flow for Airfoil problem is written as follows:
\begin{equation}
\label{eq:AirfoilEquation}
\begin{aligned} 
    \partial_t \rho+\nabla \cdot(\rho \mathbf{u}) &= f_1, \\
    \partial_t(\rho \mathbf{u})+\nabla \cdot(\rho \mathbf{u} \otimes \mathbf{u}+p \mathbb{I}) &= \mathbf{f}_2, \\
    \partial_t(\rho E)+\nabla \cdot(\rho E \mathbf{u}+p \mathbf{u}) &= f_3,
\end{aligned}
\end{equation}
where $\mathbf{u}:=\mathbf{u}(x, t)$, $\rho:=\rho(x, t)$ and $p:=p(x, t)$ are the velocity field, density and pressure respectively. $E$ denotes the total energy per unit mass. and $f_1$, $\mathbf{f}_2$, $f_3$ are generic source terms. The simulation domain includes $x \in \Omega$, $t \in[0, 4.800]$.
The dataset we use is pregenerated from~\cite{pfaff2021learning}~\footnote{https://github.com/deepmind/deepmind-research/tree/master/meshgraphnets} and preprocessed by~\cite{li_transformer_2022}~\footnote{https://github.com/BaratiLab/OFormer}.

The objective solution operator $\mathcal{S}$ is defined as follows:
\begin{equation}
\label{eq:AirfoilOperator}
\begin{aligned} 
    \mathcal{S}: \{\mathbf{u}, \rho, p\}|_{t \in[0, t^\prime]} \mapsto \{\mathbf{u}, \rho, p\}|_{t \in (t^\prime, T]}, \quad (x, y) \in \Omega,
\end{aligned}
\end{equation}
where $t^\prime = 0.576s$ and $T = 4.800s$. $\mathbf{u} := \mathbf{u}(x, y, t)$ is the velocity field, $\rho := \rho(x, y, t)$ is the density, and $p := p(x, y, t)$ is the pressure. 
Following the settings from~\cite{pfaff2021learning}, 1000/100 samples are included for training and testing, respectively. Each data sample is discretised into 5,233 irregular grid, and 101 time points, which are further temporally undersampled to 26 time points (4 for input and 22 for target).

\section{Further Details on Model Architecture \& Implementations}~\label{Appendix_Impl}
Table~\ref{tab:implementation} provides a detailed breakdown of the implementation specifics of our model trained on various datasets, including Darcy Flow 2D, Shallow Water 2D, Diffusion Reaction 2D, and Airfoil. The table also outlines the dataset settings, indicating the amount of training and testing data used, spatial and temporal resolution, and specific graph-related parameters such as radius and node feature channels, highlighting the bespoke nature of the model configuration. The dataset-specific hyperparameters follow the PDEBench \cite{takamoto2022pdebench} setting, while model-specific hyperparameters follow the default setting of baseline methods suggested by the code repositories or their papers. \\ \\
We have presented results of our \PhysGTN\ on a wide range of datasets, some of which are larger and more complex, e.g., larger spatial and temporal resolution on Diffusion Reaction 2D. On these datasets, there are no scalability issues found in our \PhysGTN\, implying that our method does not suffer from the scalability issue. To indicate computational complexity, we use the inference time\footnote{Inference time is measured as an average of 50 runs, with a batch size of 1, on an NVIDIA RTX3090. We use inference time to indicate computational complexity, since FLOPs are not constant for our proposed model, as they partly depend on the number of edges.}.We directly train \PhysGTN\ unrolling for 90 timestamps.

\begin{table}[htbp]
  \centering
  \caption{The implementation details of \PhysGTN\ }
    \begin{tabular}{c|cccc}
    \toprule
    Dataset & Darcy Flow 2D & Shallow Water 2D & Diffusion Reaction 2D & Airfoil \\
    \midrule
    \multicolumn{5}{c}{Learning Rate (LR) Parameter} \\
    \midrule
    Initial LR & 0.0001 & 0.0001 & 0.0001 & 0.0001 \\
    Schedule & OneCycleLR & OneCycleLR & OneCycleLR & OneCycleLR \\
    "div\_factor" & 20    & 20    & 20    & 20 \\
    "pct\_start" & 0.05  & 0.05  & 0.05  & 0.05 \\
    "final\_div\_factorr" & 1000  & 1000  & 1000  & 1000 \\
    \midrule
    \multicolumn{5}{c}{Optimisation Parameter} \\
    \midrule
    Optimiser & Adam  & Adam  & Adam  & Adam \\
    Data Loss Type & MSE   & MSE   & MSE   & Reletive L2 Norm \\
    Data Loss Weight & 1.0   & 1.0   & 1.0   & 1.0 \\
    \midrule
    \multicolumn{5}{c}{ Encoder (Input) Architecture} \\
    \midrule
    Hidden Dim (EncI) & 96    & 192   & 192   & 192 \\
    \#Blocks -- Graph Transformer & 10    & 6     & 6     & 6 \\
    \#Head & 8     & 4     & 4     & 4 \\
    \midrule
    \multicolumn{5}{c}{Encoder (Query) Architecture} \\
    \midrule
    Hidden Dim (EncQ) & 96    & 192   & 192   & 192 \\
    \#\#Linear Layer -- MLP & 2     & 2     & 2     & 2 \\
    \midrule
    \multicolumn{5}{c}{Decoder Architecture} \\
    \midrule
    Hidden Dim (Dec) & 256   & 512   & 512   & 512 \\
    \#Linear Layer -- Output-MLP & 2     & 3     & 3     & 3 \\
    \#Linear Layer -- Propagator-MLP & N.A.  & 3     & 3     & 3 \\
    \#Head - Cross Attention & 8     & 4     & 4     & 4 \\
    \midrule
    \multicolumn{5}{c}{Dataset Setting (Benchmark Setting)} \\
    \midrule
    \#Training Data & 9000  & 900   & 900   & 1000 \\
    \#Testing Data & 1000  & 100   & 100   & 100 \\
    Spatial Res. & 64 × 64 & 128 × 128 & 128 × 128 & N.A. \\
    Temporal Res. & N.A.  & 101   & 101   & 26 \\
    Length -- Input Seq. & N.A.  & 10    & 10    & 4 \\
    Length -- Output Seq. & N.A.  & 91    & 91    & 22 \\
    Graph - Radius  & 0.1   & 0.1   & 0.08  & 0.001 \\
    Graph - Node Feature Chnl. & 3     & 12    & 22    & 18 \\
    Number of Parameters & 3.16M & 4.29M & 4.30M & 5.16M \\
    Inference Time & 0.306s & 0.313s & 0.378s & 0.094s \\
    \bottomrule
    \end{tabular}%
  \label{tab:implementation}%
\end{table}%


\end{document}